\documentclass[10pt,english]{article}
\usepackage[T1]{fontenc}
\usepackage[latin9]{inputenc}
\usepackage{array}
\usepackage{amsmath}
\usepackage{amsthm}
\usepackage{amssymb}
\usepackage{graphicx}
\usepackage[authoryear]{natbib}

\makeatletter

\providecommand{\tabularnewline}{\\}

\theoremstyle{plain}
\newtheorem{thm}{\protect\theoremname}
\theoremstyle{definition}
\newtheorem{defn}[thm]{\protect\definitionname}
\theoremstyle{plain}
\newtheorem{prop}[thm]{\protect\propositionname}
\ifx\proof\undefined
\newenvironment{proof}[1][\protect\proofname]{\par
	\normalfont\topsep6\p@\@plus6\p@\relax
	\trivlist
	\itemindent\parindent
	\item[\hskip\labelsep\scshape #1]\ignorespaces
}{%
	\endtrivlist\@endpefalse
}
\providecommand{\proofname}{Proof}
\fi

\usepackage[accepted]{icml2021}

\icmltitlerunning{Beyond Folklore: A Scaling Calculus for the Design and Initialization of ReLU Networks}

\usepackage{amssymb}
\usepackage{pifont}
\newcommand{\xmark}{\ding{55}}%

\@ifundefined{showcaptionsetup}{}{%
 \PassOptionsToPackage{caption=false}{subfig}}
\usepackage{subfig}
\makeatother

\usepackage{babel}
\providecommand{\definitionname}{Definition}
\providecommand{\propositionname}{Proposition}
\providecommand{\theoremname}{Theorem}

\begin{document}
\twocolumn[
\icmltitle{Beyond Folklore: A Scaling Calculus for the Design and Initialization of ReLU Networks}

\begin{icmlauthorlist}
\icmlauthor{Aaron Defazio}{equal,fb}
\icmlauthor{L{\'e}on Bottou}{fb}
\end{icmlauthorlist}

\icmlaffiliation{fb}{Facebook AI Research New York}
\icmlcorrespondingauthor{Aaron Defazio}{firstname.lastname@gmail.com}
\vskip 0.3in
]

\printAffiliationsAndNotice{}
\begin{abstract}
We propose a system for calculating a ``scaling constant'' for layers
and weights of neural networks. We relate this scaling constant to
two important quantities that relate to the optimizability of neural
networks, and argue that a network that is ``preconditioned'' via
scaling, in the sense that all weights have the same scaling constant,
will be easier to train. This scaling calculus results in a number
of consequences, among them the fact that the geometric mean of the
fan-in and fan-out, rather than the fan-in, fan-out, or arithmetic
mean, should be used for the initialization of the variance of weights
in a neural network. Our system allows for the off-line design \&
engineering of ReLU neural networks, potentially replacing blind experimentation.
\end{abstract}

\section{Introduction}

The design of neural networks is often considered a black-art, driven
by trial and error rather than foundational principles. This is exemplified
by the success of recent architecture random-search techniques \citep{NAS,Li2019RandomSA},
which take the extreme of applying no human guidance at all. Although
as a field we are far from fully understanding the nature of learning
and generalization in neural networks, this does not mean that we
should proceed blindly. 

In this work, we define a scaling quantity $\gamma_{l}$ for each
layer $l$ that approximates two quantities of interest when considering
the optimization of a neural network: The ratio of the gradient to
the weights, and the average squared singular value of the corresponding
diagonal block of the Hessian for layer $l$. This quantity is easy
to compute from the (non-central) second moments of the forward-propagated
values and the (non-central) second moments of the backward-propagated
gradients. We argue that networks that have constant $\gamma_{l}$
are better conditioned than those that do not, and we analyze how
common layer types affect this quantity. We call networks that obey
this rule \emph{preconditioned} neural networks.

As an example of some of the possible applications of our theory,
we: 
\begin{itemize}
\item Propose a \emph{principled} weight initialization scheme that can
often provide an improvement over existing schemes;
\item Show which common layer types automatically result in well-conditioned
networks;
\item Show how to improve the conditioning of common structures such as
bottlenecked residual blocks by the addition of fixed scaling constants
to the network.
\end{itemize}

\section{Notation}

Consider a neural network mapping $x_{0}$ to $x_{L}$ made up of
$L$ layers. These layers may be individual operations or blocks of
operations. During training, a loss function is computed for each
minibatch of data, and the gradient of the loss is back-propagated
to each layer $l$ and weight of the network. We prefix each quantity
with $\Delta$ to represent the back-propagated gradient of that quantity.
We assume a batch-size of 1 in our calculations, although all conclusions
hold using mini-batches as well.

Each layer's input activations are represented by a tensor $x_{l}:n_{l}\times\rho_{l}\times\rho_{l}$
made up of $n_{l}$ channels, and spatial dimensions $\rho_{l}\times\rho_{l}$,
assumed to be square for simplicity (results can be adapted to the
rectangular case by using $h_{l}w_{l}$ in place of $\rho_{l}$ everywhere). 

\section{A model of ReLU network dynamics}

\label{sec:assumptions}Our scaling calculus requires the use of simple
approximations of the dynamics of neural networks, in the same way
that simplifications are used in physics to make approximate calculations,
such as the assumption of zero-friction or ideal gasses. These assumptions
constitute a model of the behavior of neural networks that allows
for easy calculation of quantities of interest, while still being
representative enough of the real dynamics.

To this end, we will focus in this work on the behavior of networks
at initialization. Furthermore, we will make strong assumptions on
the statistics of forward and backward quantities in the network.
These assumptions include:
\begin{enumerate}
\item The input to layer $l$, denoted $x_{l}$, is a random tensor assumed
to contain i.i.d entries. We represent the element-wise uncentered
2nd moment by $E[x_{l}^{2}]$.
\item The back-propagated gradient of $x_{l}$ is $\Delta x_{l}$ and is
assumed to be uncorrelated with $x_{l}$ and iid. We represent the
uncentered 2nd-moment of $\Delta x_{l}$ by $E[\Delta x_{l}^{2}]$.
\item All weights in the network are initialized i.i.d from a centered,
symmetric distribution.
\item All bias terms are initialized as zero.
\end{enumerate}
Our calculations rely heavily on the uncentered second moments rather
than the variance of weights and gradients. This is a consequence
of the behavior of the ReLU activation, which zeros out entries. The
effect of this zeroing operation is simple when considering uncentered
second moments under a symmetric input distribution, as half of the
entries will be zeroed, resulting in a halving of the uncentered second
moment. In contrast, expressing the same operation in terms of variance
is complicated by the fact that the mean after application of the
ReLU is distribution-dependent. We will refer to the uncentered second
moment just as the ``second moment'' henceforth.

\section{Activation and layer scaling factors}
\begin{defn}
The key quantity in our calculus is the activation scaling factor
$\varsigma_{l}$, of the input activations for a layer $l$, which
we define as:
\begin{equation}
\varsigma_{l}=n_{l}\rho_{l}^{2}E[\Delta x_{l}^{2}]E[x_{l}^{2}].\label{eq:actscale}
\end{equation}
\end{defn}
This quantity arises due to its utility in computing other quantities
of interest in the network, such as the scaling factors for the weights
of convolutional and linear layers. In ReLU networks, many, but not
all operations maintain this quantity in the sense that $\varsigma_{l}=\varsigma_{l+1}$
for a layer $x_{l+1}=F(x_{l})$ with operation $F$, under the assumptions
of Section \ref{sec:assumptions}. Table \ref{tab:scalingtable} contains
a list of common operations and indicates if they maintain scaling.
As an example, consider adding a simple scaling layer of the form
$x_{l+1}=\sqrt{2}x_{l}$ which doubles the second moment during the
forward pass and doubles the backward second moment during back-propagation.
We can see that:
\begin{align*}
\varsigma_{l+1} & =n_{l+1}\rho_{l+1}^{2}E[\Delta x_{l+1}^{2}]E[x_{l+1}^{2}]\\
 & =n_{l}\rho_{l}^{2}\frac{1}{2}E[\Delta x_{l}^{2}]\cdot2E[x_{l}^{2}]=\varsigma_{l}
\end{align*}

Our analysis in our work is focused on ReLU networks primarily due
to the fact that ReLU non-linearities maintain this scaling factor.

\begin{table*}
\noindent \begin{centering}
\caption{\label{tab:scalingtable}Scaling of common layers}
\par\end{centering}
\noindent \centering{}{\small{}}%
\begin{tabular}{|c|c|>{\centering}p{0.55\textwidth}|}
\hline 
{\small{}Method} & {\small{}Maintains Scaling} & {\small{}Notes}\tabularnewline
\hline 
\hline 
{\small{}Linear layer} & {\small{}\checkmark} & \noindent \raggedright{}{\small{}Layer scaling requires geometric
initialization}\tabularnewline
\hline 
{\small{}(Strided) convolution} & {\small{}\checkmark} & {\small{}Requires stride equal to the kernel size}\tabularnewline
\hline 
{\small{}Skip connections} & {\small{}\checkmark} & {\small{}Operations within residual blocks will also be scaled correctly
against other residual blocks, but not against outside operations.}\tabularnewline
\hline 
{\small{}Average pooling} & {\small{}\checkmark} & {\small{}Requires stride equal to the kernel size}\tabularnewline
\hline 
{\small{}Max pooling} & {\small{}\xmark} & \tabularnewline
\hline 
{\small{}Dropout} & {\small{}\checkmark} & \tabularnewline
\hline 
{\small{}ReLU/LeakyReLU} & {\small{}\checkmark} & {\small{}Any positively-homogenous function with degree 1}\tabularnewline
\hline 
{\small{}Sigmoid} & {\small{}\xmark} & \tabularnewline
\hline 
{\small{}Tanhh} & {\small{}\xmark} & {\small{}Maintains scaling if entirely within the linear regime}\tabularnewline
\hline 
\end{tabular}\vspace{-0.5em}
\end{table*}

\begin{defn}
Using the activation scaling factor, we define the layer or weight
scaling factor of a convolutional layer with kernel $k_{l}\times k_{l}$
as:
\begin{equation}
\gamma_{l}=\frac{\varsigma_{l}}{n_{l+1}n_{l}k_{l}^{2}E[W_{l}^{2}]^{2}}.\label{eq:sf-intrinsic}
\end{equation}
\end{defn}
Recall that $n_{l}$ is the fan-in and $n_{l+1}$ is the fan-out of
the layer. This expression also applies to linear layers by taking
$k_{l}=1$. This quantity can also be defined extrinsically without
reference to the weight initialization via the expression:
\[
\gamma_{l}=n_{l}k_{l}^{2}\rho_{l}^{2}E\left[x_{l}^{2}\right]^{2}\frac{E[\Delta x_{l+1}^{2}]}{E[x_{l+1}^{2}]}.
\]
we establish this equivalence under the assumptions of Section \ref{sec:assumptions}
in the appendix.

\section{Motivations for scaling factors}

We can motivate the utility of our scaling factor definition by comparing
it to another simple quantity of interest. For each layer, consider
the ratio of the second moments between the weights, and their gradients:
\[
\nu_{l}\doteq\frac{E[\Delta W_{l}^{2}]}{E[W_{l}^{2}]}.
\]
This ratio approximately captures the relative change that a single
SGD step with unit step-size on $W_{l}$ will produce. We call this
quantity the weight-to-gradient ratio. When $E[\Delta W_{l}^{2}]$
is very small compared to $E[W_{l}^{2}]$, the weights will stay close
to their initial values for longer than when $E[\Delta W_{l}^{2}]$
is large. In contrast, if $E[\Delta W_{l}^{2}]$ is very large compared
to $E[W_{l}^{2}]$, then learning can be expected to be unstable,
as the sign of the elements of $W$ may change rapidly between optimization
steps. A network with constant $\nu_{l}$ is also well-behaved under
weight-decay, as the ratio of weight-decay second moments to gradient
second moments will stay constant throughout the network, keeping
the push-pull of gradients and decay constant across the network.
This ratio also captures a relative notion of exploding or vanishing
gradients. Rather than consider if the gradient is small or large
in absolute value, we consider its relative magnitude instead.
\begin{prop}
The weight to gradient ratio $\nu_{l}$ is equal to the scaling factor
$\gamma_{l}$ under the assumptions of Section \ref{sec:assumptions}
\@.
\end{prop}

\subsection{Conditioning of the Hessian}

The scaling factor of a layer $l$ is also closely related to the
singular values of the diagonal block of the Hessian corresponding
to that layer. We derive a correspondence in this section, providing
further justification for our definition of the scaling factor above.
We focus on non-convolutional layers for simplicity in this section,
although the result extends to the convolutional case without issue.

ReLU networks have a particularly simple structure for the Hessian
for any set of activations, as the network's output is a piecewise-linear
function $g$ fed into a final layer consisting of a loss. This structure
results in greatly simplified expressions for diagonal blocks of the
Hessian with respect to the weights, and allows us to derive expressions
involving the singular values of these blocks.

We will consider the output of the network as a composition of two
functions, the current layer $g$, and the remainder of the network
$h$. We write this as a function of the weights, i.e. $f(W_{l})=h(g(W_{l}))$.
The dependence on the input to the network is implicit in this notation,
and the network below layer $l$ does not need to be considered.

Let $R_{l}=\nabla_{x_{l+1}}^{2}h(x_{l+1})$ be the Hessian of $h$,
the remainder of the network after application of layer $l$ (For
a linear layer $x_{l+1}=W_{l}x_{l}$). Let $J_{l}$ be the Jacobian
of $y_{l}$ with respect to $W_{l}$. The Jacobian has shape $J_{l}:n_{l}^{\text{out}}\times\left(n_{l}^{\text{out}}n_{l}^{\text{in}}\right)$.
Given these quantities, the diagonal block of the Hessian corresponding
to $W_{l}$ is equal to:
\[
G_{l}=J_{l}^{T}R_{l}J_{l}.
\]
The \emph{$l$th diagonal block of the (Generalized) Gauss-Newton}
matrix $G$ \citep{martens-insights}. We discuss this decomposition
further in the appendix. 

Assume that the input-output Jacobian $\Phi$ of the remainder of
the network above each block is initialized so that $\left\Vert \Phi\right\Vert _{2}^{2}=O(1)$
with respect to $n_{l+1}$. This assumption just encodes the requirement
that initialization used for the remainder of the network is sensible,
so that the output of the network does not blow-up for large widths.
\begin{prop}
Under the assumptions outlined in Section \ref{sec:assumptions},
for linear layer $l$, the average squared singular value of $G_{l}$
is equal to:

\[
n_{l}E\left[x_{l}^{2}\right]^{2}\frac{E[\Delta x_{l+1}^{2}]}{E[x_{l+1}^{2}]}+O\left(\frac{n_{l}E\left[x_{l}^{2}\right]^{2}}{n_{l+1}E[x_{l+1}^{2}]}\right).
\]
The Big-O term is with respect to $n_{l}$ and $n_{l+1}$; its precise
value depends on properties of the remainder of the network above
the current layer.
\end{prop}
\begin{figure*}
\begin{centering}
\includegraphics[width=1\textwidth]{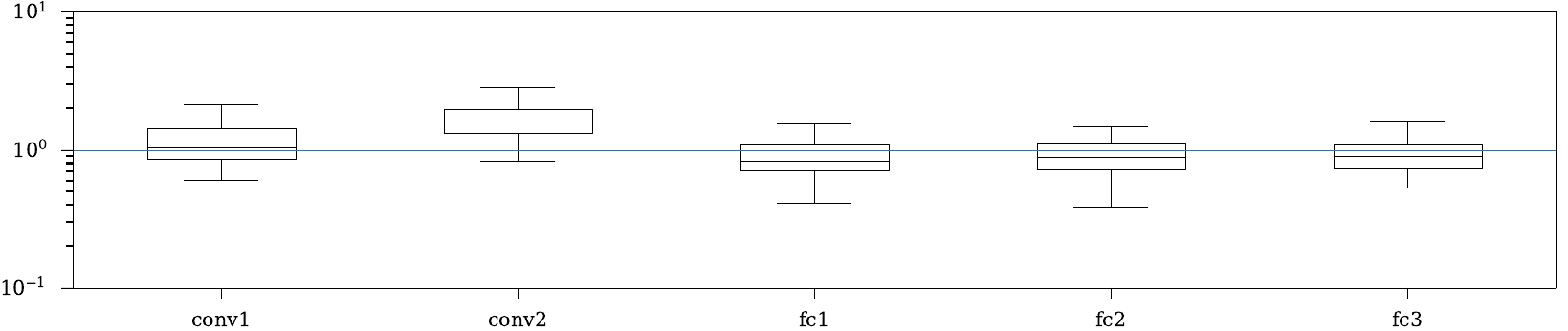}
\par\end{centering}
\caption{\label{fig:ratio-lenet}Distributions of the ratio of theoretical
scaling to actual for a strided LeNet network. The ratios are close
to the ideal value of 1, indicating good theoretical and practical
agreement.}
\end{figure*}
Despite the approximations required for its derivation, the scaling
factor can still be close to the actual average squared singular value.
We computed the ratio of the scaling factor (Equation \ref{eq:sf-intrinsic})
to the actual expectation $E[\left(G_{l}r\right)^{2}]$ for a strided
(rather than max-pooled, see Table \ref{tab:scalingtable}) LeNet
model, where we use random input data and a random loss (i.e. for
outputs $y$ we use $y^{T}Ry$ for an i.i.d normal matrix $R$), with
batch-size 1024, and $32\times32$ input images. The results are shown
in Figure \ref{fig:ratio-lenet} for 100 sampled setups; there is
generally good agreement with the theoretical expectation.

\section{Initialization of ReLU networks}

\label{sec:precond-geom}An immediate consequence of our definition
of the scaling factor is a rule for the initialization of ReLU networks.
Consider a network where the activation scaling factor is constant
through-out. Then any two layers $l$ and $r$ will have the same
weight scaling factor if $\gamma_{l}=\gamma_{r}$, which holds immediately
when each layer is initialized with: 
\begin{equation}
E[W_{l}^{2}]=\frac{c}{k_{l}\sqrt{n_{l}n_{l+1}}},\label{eq:geom-init}
\end{equation}
for some fixed constant $c$ independent of the layer. Initialization
using the geometric-mean of the fan-in and fan-out ensures a constant
layer scaling factor throughout the network, aiding optimization.
Notice that the dependence on the kernel size is also unusual, rather
than $k_{l}^{2}$, we normalize by $k_{l}$. 

\subsection{Other Initialization schemes}

The most common approaches are the \emph{Kaiming} \citep{kaiming-rectifiers2015}
(sometimes called He) and \emph{Xavier} \citep{xavierglorot2010}
(sometimes called Glorot) initializations. The Kaiming technique for
ReLU networks is one of two approaches:
\begin{equation}
(\text{fan-in)}\quad\text{Var}[W_{l}]=\frac{2}{n_{l}k_{l}^{2}}\;\text{or}\label{eq:fan-in}
\end{equation}
\[
(\text{fan-out)}\quad\text{Var}[W_{l}]=\frac{2}{n_{l+1}k_{l}^{2}}
\]
For the feed-forward network above, assuming random activations, the
forward-activation variance will remain constant in expectation throughout
the network if fan-in initialization of weights \citep{efficientbackprop2012}
is used, whereas the fan-out variant maintains a constant variance
of the back-propagated signal. The constant factor 2 corrects for
the variance-reducing effect of the ReLU activation. Although popularized
by \citet{kaiming-rectifiers2015}, similar scaling was in use in
early neural network models that used tanh activation functions \citep{bottou-88b}. 

These two principles are clearly in conflict; unless $n_{l}=n_{l+1}$,
either the forward variance or backward variance will become non-constant.
No \emph{prima facie} reason for preferring one initialization over
the other is provided. Unfortunately, there is some confusion in the
literature as many works reference using Kaiming initialization without
specifying if the fan-in or fan-out variant is used.

The Xavier initialization \citep{xavierglorot2010} is the closest
to our proposed approach. They balance these conflicting objectives
using the arithmetic mean:
\begin{equation}
\text{Var}[W_{l}]=\frac{2}{\frac{1}{2}\left(n_{l}+n_{l+1}\right)k_{l}^{2}},\label{eq:xavier}
\end{equation}
to ``... approximately satisfy our objectives of maintaining activation
variances and back-propagated gradients variance as one moves up or
down the network''. This approach to balancing is essentially heuristic,
in contrast to the geometric mean approach that our theory directly
guides us to.

Figure \ref{fig:lenet_heatmaps} shows heat maps of the average singular
values for each block of the Hessian of a LeNet model under the initializations
considered. The use of geometric initialization results in an equally
weighted diagonal, in contrast to the other initializations considered.

\begin{figure*}
\hfill{}\subfloat[Fan-in]{\includegraphics[scale=1.25]{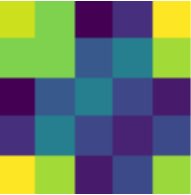}

}\hfill{}\subfloat[Fan-out]{\includegraphics[scale=1.25]{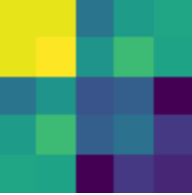}

}\hfill{}\subfloat[Arithmetic mean]{\includegraphics[scale=1.25]{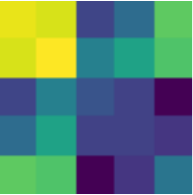}

}\hfill{}\subfloat[Geometric mean]{\includegraphics[scale=1.25]{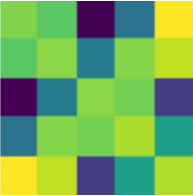}

}\hfill{}

\caption{\label{fig:lenet_heatmaps}Average singular value heat maps for the
strided LeNet model, where each square represents a block of the Hessian,
with blocking at the level of weight matrices (biases omitted). Using
geometric initialization maintains an approximately constant block-diagonal
weight. The scale goes from Yellow (larger) through green to blue
(smaller).}
\end{figure*}

\subsection{Practical application}

\label{subsec:dep_kernel}The dependence of the geometric initialization
on kernel size rather than its square will result in a large increase
in forward second moments if $c$ is not carefully chosen. We recommend
setting $c=2/k$, where $k$ is the typical kernel size in the network.
Any other layer in the network with kernel size differing from this
default should be preceded by a fixed scaling factor $x_{l+1}=\alpha x_{l}$,
that corrects for this. For instance, if the typical kernel size is
1, then a 3x3 convolution would be preceded with a $\alpha=\sqrt{1/3}$
fixed scaling factor.

In general, we have the freedom to modify the initialization of a
layer, then apply a fixed multiplier before or after the layer to
``undo'' the increase. This allows us to change the behavior of
a layer during learning by modifying the network rather than modifying
the optimizer. Potentially, we can avoid the need for sophisticated
adaptive optimizers by designing networks to be easily optimizable
in the first place. In a sense, the need for maintaining the forward
or backward variance that motivates that fan-in/fan-out initialization
can be decoupled from the choice of initialization, allowing us to
choose the initialization to improve the optimizability of the network.

Initialization by the principle of dynamical isometry \citep{nonlineardyn2014,dynamicalisometry2018},
a form of orthogonal initialization \citep{mishkin2016need} has been
shown to allow for the training of very deep networks. Such orthogonal
initializations can be combined with the scaling in our theory without
issue, by ensuring the input-output second moment scaling is equal
to the scaling required by our theory. Our analysis is concerned with
the correct initialization when layer widths change within a network,
with is a separate concern from the behavior of a network in a large-depth
limit, where all layers are typically taken to be the same width.
In ReLU networks orthogonal initialization is less interesting, as
``... the ReLU nonlinearity destroys the qualitative scaling advantage
that linear networks possess for orthogonal weights versus Gaussian''
\citet{spectraluni2018}.

\section{Output second moments}

\label{subsec:output_norm}A neural network's behavior is also very
sensitive to the second moment of the outputs. We are not aware of
any existing theory guiding the choice of output variance at initialization
for the case of log-softmax losses, where it has a non-trivial effect
on the back-propagated signals, although output variances of 0.01
to 0.1 are reasonable choices to avoid saturating the non-linearity
while not being too close to zero. The output variance should \textbf{\small{}always}
be checked and potentially corrected when switching initialization
schemes, to avoid inadvertently large or small values. 

In general, the variance at the last layer may easily be modified
by inserting a fixed scalar multiplier $x_{l+1}=\alpha x_{l}$ anywhere
in the network, and so we have complete control over this variance
independently of the initialization used. For a simple ReLU convolutional
network with all kernel sizes the same, and without pooling layers
we can compute the output second moment when using geometric-mean
initialization ($c=2/k$) with the expression:
\begin{align}
E[x_{l+1}^{2}] & =\frac{1}{2}k_{l}^{2}n_{l}E[W_{l}^{2}]E[x_{l}^{2}]=\sqrt{\frac{n_{l}}{n_{l+1}}}E[x_{l}^{2}].\label{eq:n_change}
\end{align}
The application of a sequence of these layers gives a telescoping
product:
\begin{align*}
E[x_{L}^{2}] & =\left(\prod_{l=0}^{L-1}\sqrt{\frac{n_{l}}{n_{l+1}}}\right)E[x_{0}^{2}]=\sqrt{\frac{n_{0}}{n_{L}}}E[x_{0}^{2}].
\end{align*}
so the output variance is independent of the interior structure of
the network and depends only on the input and output channel sizes. 

\section{Biases}

The conditioning of the additive biases in a network is also crucial
for learning. Since our model requires that biases be initialized
to zero, we can not use the gradient to weight ratio for capturing
the conditioning of the biases in the network. The average singular
value notion of conditioning still applies, which leads to the following
definition:
\begin{defn}
The scaling of the bias of a layer $l$, $x_{l+1}=C_{l}(x_{l})+b_{l}$
is defined as:
\begin{equation}
\gamma_{l}^{b}=\rho^{2}\frac{E[\Delta x_{l}^{2}]}{E[x_{l}^{2}]}.\label{eq:bias_scaling}
\end{equation}
In terms of the activation scaling this is:
\begin{align}
\gamma_{l}^{b} & =\rho^{2}\frac{E[\Delta x_{l}^{2}]}{E[x_{l}^{2}]}\nonumber \\
 & =\frac{\varsigma_{l}}{nE[\Delta x_{l}^{2}]E[x_{l}^{2}]}\frac{E[\Delta x_{l}^{2}]}{E[x_{l}^{2}]}\nonumber \\
 & =\frac{\varsigma_{l}}{n_{l}E[x_{l}^{2}]^{2}}.\label{eq:bias_scale_ex}
\end{align}
From Equation \ref{eq:n_change} it's clear that when geometric initialization
is used with $c=2/k$, then:
\[
n_{l+1}E[x_{l+1}^{2}]^{2}=n_{l}E[x_{l}^{2}]^{2},
\]
and so all bias terms will be equally scaled against each other. If
kernel sizes vary in the ReLU network, then a setting of $c$ following
Section \ref{subsec:dep_kernel} should be used, combined with fixed
scalar multipliers that ensure that at initialization $E[x_{l+1}^{2}]=\sqrt{\frac{n_{l}}{n_{l+1}}}E[x_{l}^{2}]$.
\end{defn}

\subsection{Network input scaling balances weights against biases}

It is traditional to normalize a dataset before applying a neural
network so that the input vector has mean 0 and variance 1 in expectation.
This scaling originated when neural networks commonly used sigmoid
and tanh nonlinearities, which depended heavily on the input scaling.
This principle is no longer questioned today, even though there is
no longer a good justification for its use in modern ReLU based networks.
In contrast, our theory provides direct guidance for the choice of
input scaling. 

Consider the scaling factors for the bias and weight parameters in
the first layer of a ReLU-based network, as considered in previous
sections. We assume the data is already centered. Then the scaling
factors for the weight and bias layers are:
\[
\gamma_{0}=n_{0}k_{0}^{2}\rho_{1}^{2}E\left[x_{0}^{2}\right]^{2}\frac{E[\Delta y_{0}^{2}]}{E[y_{0}^{2}]},\qquad\gamma_{0b}=\rho_{1}^{2}\frac{E[\Delta y_{0}^{2}]}{E[y_{0}^{2}]}.
\]
We can cancel terms to find the value of $E\left[x_{0}^{2}\right]$
that makes these two quantities equal:
\[
E\left[x_{0}^{2}\right]=\frac{1}{\sqrt{n_{0}k_{0}^{2}}}.
\]
In common computer vision architectures, the input planes are the
3 color channels and the kernel size is $k=3$, giving $E\left[x_{0}^{2}\right]\approx0.2$.
Using the traditional variance-one normalization will result in the
effective learning rate for the bias terms being lower than that of
the weight terms. This will result in potentially slower learning
of the bias terms than for the input scaling we propose. We recommend
including an initial forward scaling factor in the network of $1/(n_{0}k^{2})^{1/4}$
to correct for this.

\section{Experimental Results on 26 LIBSVM datasets}

\begin{table*}
\caption{\label{tab:comp}Comparison on 26 LIBSVM repository datasets\vspace{-1em}
}

\noindent \centering{}%
\begin{tabular}{|>{\raggedright}p{0.2\textwidth}|>{\centering}p{0.35\textwidth}|>{\centering}p{0.15\textwidth}|>{\centering}p{0.15\textwidth}|}
\hline 
Method & Average Normalized loss $(\pm0.01)$ & Worst in \# & Best in \#\tabularnewline
\hline 
\hline 
Arithmetic mean & 0.90 & 14 & 3\tabularnewline
\hline 
Fan-in & 0.84 & 3 & 5\tabularnewline
\hline 
Fan-out & 0.88 & 9 & \textbf{12}\tabularnewline
\hline 
Geometric mean & \textbf{0.81} & \textbf{0} & 6\tabularnewline
\hline 
\end{tabular}
\end{table*}
We considered a selection of dense and moderate-sparsity multi-class
classification datasets from the LibSVM repository, 26 in total, collated
from a variety of sources \citep{Dua:2019,smallnorb,statlog,usps,mnist,news20,svmguide,SVHN,aloi,sensor,protein,cifar10}.
The same model was used for all datasets, a non-convolutional ReLU
network with 3 weight layers total. The inner-two layer widths were
fixed at 384 and 64 nodes respectively. These numbers were chosen
to result in a larger gap between the optimization methods, less difference
could be expected if a more typical $2\times$ gap was used. Our results
are otherwise generally robust to the choice of layer widths.

For every dataset, learning rate, and initialization combination we
ran 10 seeds and picked the median loss after 5 epochs as the focus
of our study (The largest differences can be expected early in training).
Learning rates in the range $2^{1}$ to $2^{-12}$ (in powers of 2)
were checked for each dataset and initialization combination, with
the best learning rate chosen in each case based on the median of
the 10 seeds. Training loss was used as the basis of our comparison
as we care primarily about convergence rate, and are comparing identical
network architectures. Some additional details concerning the experimental
setup and which datasets were used are available in the appendix.

Table \ref{tab:scalingtable} shows that \textbf{geometric initialization
is the most consistent} of the initialization approaches considered.\textbf{
It has the lowest loss}, after normalizing each dataset, and it is
never the worst of the 4 methods on any dataset. Interestingly, the
fan-out method is most often the best method, but consideration of
the per-dataset plots (Appendix \ref{sec:forward-backward}) shows
that it often completely fails to learn for some problems, which pulls
up its average loss and results in it being the worst for 9 of the
datasets.

\section{Convolutional case: AlexNet experiments}

To provide a clear idea of the effect of our scaling approach on larger
networks we used the AlexNet architecture \citep{krizhevsky2012imagenet}
as a test bench. This architecture has a large variety of filter sizes
(11, 5, 3, linear), which according to our theory will affect the
conditioning adversely, and which should highlight the differences
between the methods. The network was modified to replace max-pooling
with striding as max-pooling is not well-scaled by our theory.
\begin{figure}
\centering{}\includegraphics[width=1\columnwidth]{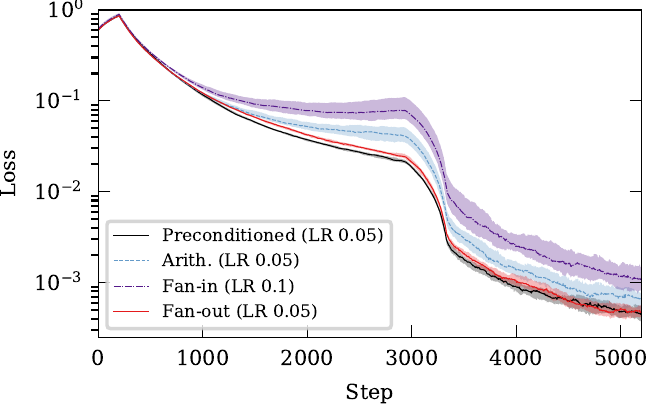}\caption{\label{fig:CIFAR10}CIFAR-10 training loss for a strided AlexNet architecture.
The median as well as a 25\%-75\% IQR of 40 seeds is shown for each
initialization, where for each seed a sliding window of minibatch
training loss over 400 steps is used.}
\end{figure}

\begin{flushleft}
Following Section \ref{subsec:output_norm}, we normalize the output
of the network at initialization by running a single batch through
the network and adding a fixed scaling factor to the network to produce
output standard deviation $0.05$. We tested on CIFAR-10 following
the standard practice as closely as possible, as detailed in the appendix.
We performed a geometric learning rate sweep over a power-of-two grid.
Results are shown in Figure \ref{fig:CIFAR10} for an average of 40
seeds for each initialization. Preconditioning is a statistically
significant improvement ($p=3.9\times10^{-6})$ over arithmetic mean
initialization and fan-in initialization, however, it only shows an
advantage over fan-out at mid-iterations. 
\par\end{flushleft}

\section{Case study: Unnormalized residual networks}

In the case of more complex network architectures, some care needs
to be taken to produce well-scaled neural networks. We consider in
this section the example of a residual network, a common architecture
in modern machine learning. Consider a simplified residual architecture
like the following, where we have omitted ReLU operations for our
initial discussion:
\begin{align*}
x_{1} & =C_{0}(x_{0}),\\
x_{2} & =B_{0}(x_{1},\alpha_{0},\beta_{0}),\\
x_{3} & =B_{1}(x_{2},\alpha_{1},\beta_{1}),\\
x_{4} & =AvgPool(x_{3}),\\
x_{5} & =L(x_{4}).
\end{align*}
where for some sequence of operations $F$:
\[
B(x,\alpha,\beta)=\alpha x+\beta F(x),
\]
we further assume that $\alpha^{2}+\beta^{2}=1$ and that $E[F(x)^{2}]=E[x^{2}]$
following \citet{rescalenet}. The use of weighted residual blocks
is necessary for networks that do not use batch normalization \citep{fixup,inceptionv4,boris2018,effectinitarch2018}.

If geometric initialization is used, then $C_{0}$ and $L$ will have
the same scaling, however, the operations within the residual blocks
will not. To see this, we can calculate the activation scaling factor
within the residual block. We define the shortcut branch for the residual
block as the $\alpha x$ operation and the main branch as the $C(x)$
operation. Let $x_{R}=\beta C(x)$ and $x_{S}=\alpha x$, and define
$y=x_{S}+x_{R}$. 

Let $\varsigma$ be the scaling factor at $x$
\[
\varsigma=n\rho^{2}E[\Delta x^{2}]E[x^{2}],
\]
We will use the fact that:
\begin{align*}
E[\Delta x^{2}] & =\left(\alpha^{2}+\beta^{2}\right)E[\Delta y^{2}]\\
 & =E[\Delta y^{2}].
\end{align*}
From rewriting the scale factor for $x_{R}$, we see that:
\begin{align*}
\varsigma_{R} & =n\rho^{2}E[\Delta x_{R}^{2}]E[x_{R}^{2}]\\
 & =n\rho^{2}E[\Delta x^{2}]E[x_{R}^{2}]\\
 & =\beta^{2}n_{l}\rho_{l}^{2}E[\Delta x^{2}]E[x^{2}]\\
 & =\beta^{2}\varsigma.
\end{align*}
A similar calculation shows that the residual branch's scaling factor
is multiplied by $\alpha^{2}$. To ensure that convolutions within
the main branch of the residual block have the same scaling as those
outside the block, we must multiply their initialization by a factor
$c$. We can calculate the value of $c$ required when geometric scaling
is used for an operation in layer $l$ in the main branch:
\begin{align*}
\gamma_{l} & =\frac{\varsigma_{R}}{n_{l+1}n_{l}k_{l}^{2}E[W_{l}^{2}]^{2}}\\
 & =\frac{\varsigma_{R}k_{l}^{2}n_{l+1}n_{l}}{n_{l+1}n_{l}k_{l}^{2}c_{l}^{2}}=\varsigma_{R}/c_{l}^{2}
\end{align*}
For $\gamma_{l}$ to match $\gamma$ outside the block we thus need
$\gamma_{l}=\varsigma_{R}/a_{l}^{2}=(\beta^{2}/c^{2})\varsigma_{R},$
i.e. $c=\beta$. If the residual branch uses convolutions (such as
for channel widening operations or down-sampling as in a ResNet-50
architecture) then they should be scaled by $\alpha$. Modifying the
initialization of the operations within the block changes $E[F(x)^{2}],$
so a fixed scalar multiplier must be introduced within the main branch
to undo the change, ensuring $E[F(x)^{2}]=E[x^{2}]$.

\subsection{Design of a pre-activation ResNet block}

Using the principle above we can modify the structure of a standard
pre-activation ResNet block to ensure all convolutions are well-conditioned
both across blocks and against the initial and final layers of the
network. We consider the full case now, where the shortcut path may
include a convolution that changes the channel count or the resolution.
Consider a block of the form:
\[
B(x,\alpha,\beta)=\alpha S(x)+\beta F(x)
\]
We consider a block with fan-in $n$ and fan-out $m$. There are two
cases, depending on if the block is a downsampling block or not. In
the case of a downsampling block, a well-scaled shortcut branch consists
of the following sequence of operations:
\begin{align*}
y_{0} & =\text{AvgPool2D}(x,\text{kernel\_size=2},\text{stride=}2),\\
y_{1} & =y_{0}+b_{0},\\
y_{2} & =C(y_{1},\text{op=}m,\text{ks=}1,c=\text{\ensuremath{\alpha}/4}),\\
y_{3} & =y_{2}/\sqrt{\alpha/4}.
\end{align*}
In our notation, $C$ is initialized with the geometric initialization
scheme of Equation \ref{eq:geom-init} using numerator $c=\alpha$.
Here $op$ is output planes and $ks$ is the kernel size. The constant
$4$ corrects for the downsampling, and the constant $\alpha$ is
used to correct the scaling factor of the convolution as described
above. In the non-downsampled case, this simplifies to
\begin{align*}
y_{0} & =x+b_{0},\\
y_{1} & =C(y_{0},\text{op=}m,\text{ks=}1,c=\text{\ensuremath{\alpha}}),\\
y_{2} & =y_{1}/\sqrt{\alpha}.
\end{align*}
For the main branch of a bottlenecked residual block in a pre-activation
network, the sequence begins with a single scaling operation $x_{0}=\sqrt{\beta}x$,
the following pattern is used, with $w$ being inner bottleneck width.
\begin{align*}
x_{1} & =\text{ReLU}(x_{0})\\
x_{2} & =\sqrt{\frac{2}{\beta}}x_{1}\\
x_{3} & =x_{2}+b_{1}\\
x_{4} & =C(x_{3},\text{op=}w,\text{ks=}1,c=\text{\ensuremath{\beta}})
\end{align*}
Followed by a 3x3 conv:
\begin{align*}
x_{5} & =\text{ReLU}(x_{4})\\
x_{6} & =\sqrt{\frac{2}{3\beta}}x_{5}\\
x_{7} & =x_{6}+b_{2}\\
x_{8} & =C(x_{7},\text{op=}w,\text{ks=}3,c=\text{\ensuremath{\beta}})
\end{align*}
and the final sequence of operations mirrors the initial operation
with a downscaling convolution instead of upscaling. At the end of
the block, a learnable scalar $x_{9}=\frac{v}{\sqrt{\beta}}x_{8}$
with $v=\sqrt{\beta}$ is included following the approach of \citet{rescalenet},
and a fixed scalar corrects for any increases in the forward second
moment from the entire sequence, in this case, $x_{9}=\sqrt{\frac{m}{\beta n}}x_{8}$.
This scaling is derived from Equation \ref{eq:n_change} ($\beta$
here undoes the initial beta from the first step in the block).

\subsection{Experimental results}

We ran a series of experiments on an unnormalized pre-activation ResNet-50
architecture using our geometric initialization and scaling scheme
both within and outside of the blocks. We compared against the RescaleNet
unnormalized ResNet-50 architecture. Following their guidelines, we
added dropout which is necessary for good performance and used the
same $\alpha/\beta$ scheme that they used. Our implementation is
available in the supplementary material. We performed our experiments
on the ImageNet dataset \citep{imagenet}, using standard data preprocessing
pipelines and hyper-parameters. In particular, we use batch-size 256,
decay 0.0001, momentum 0.9, and learning rate 0.1 with SGD, using
a 30-60-90 decreasing scheme for 90 epochs. Following our recommendation
in Section \ref{subsec:output_norm}, we performed a sweep on the
output scaling factor and found that a 0.05 final scalar gives the
best results. Across 5 seeds, our approach achieved a test set accuracy
of \textbf{76.18 (SE 0.04}), which matches the performance of the
RescaleNet within our test framework of 76.13 (SE 0.03). Our approach
supersedes the ``fixed residual scaling'' that they propose as a
way of balancing the contributions of each block.

\section{Related Work}

Our approach of balancing the diagonal blocks of the Gauss-Newton
matrix has close ties to a large literature studying the input-output
Jacobian of neural networks. The Jacobian is the focus of study in
a number of ways. The singular values of the Jacobian are the focus
of theoretical study in \citet{nonlineardyn2014,dynamicalisometry2018},
where it's shown that orthogonal initializations better control the
spread of the spectrum compared to Gaussian initializations. \citet{hanin2020,hanin2020nonasymptotic}
also study the effect of layer width and depth on the spectrum. Regularization
of the jacobian, where additional terms are added to the loss to minimize
the Frobenius norm of the Jacobian, can be seen as another way to
control the spectrum \citep{hoffman2019robust,varga2018gradient},
as the Frobenius norm is the sum of the squared singular values. The
spectrum of the Jacobian captures the sensitivity of a network to
input perturbations and is key to the understanding of adversarial
machine learning, including generative modeling \citep{Nie2019UAI}
and robustness \citep{jacadv2018,jacreg2020}. 

\section{Conclusion}

Although not a panacea, by using the scaling principle we have introduced,
neural networks can be designed with a reasonable expectation that
they will be optimizable by stochastic gradient methods, minimizing
the amount of guess-and-check neural network design. Our approach
is a step towards ``engineering'' neural networks, where aspects
of the behavior of a network can be studied in an off-line fashion
before use, rather than by a guess-implement-test-and-repeat experimental
loop.

\bibliographystyle{icml2021}
\bibliography{scaling_calculus}
\cleardoublepage{}

\appendix

\section{Forward and backward second moments}

\label{sec:forward-backward}
\begin{figure*}
\centering{}\includegraphics[width=1\textwidth]{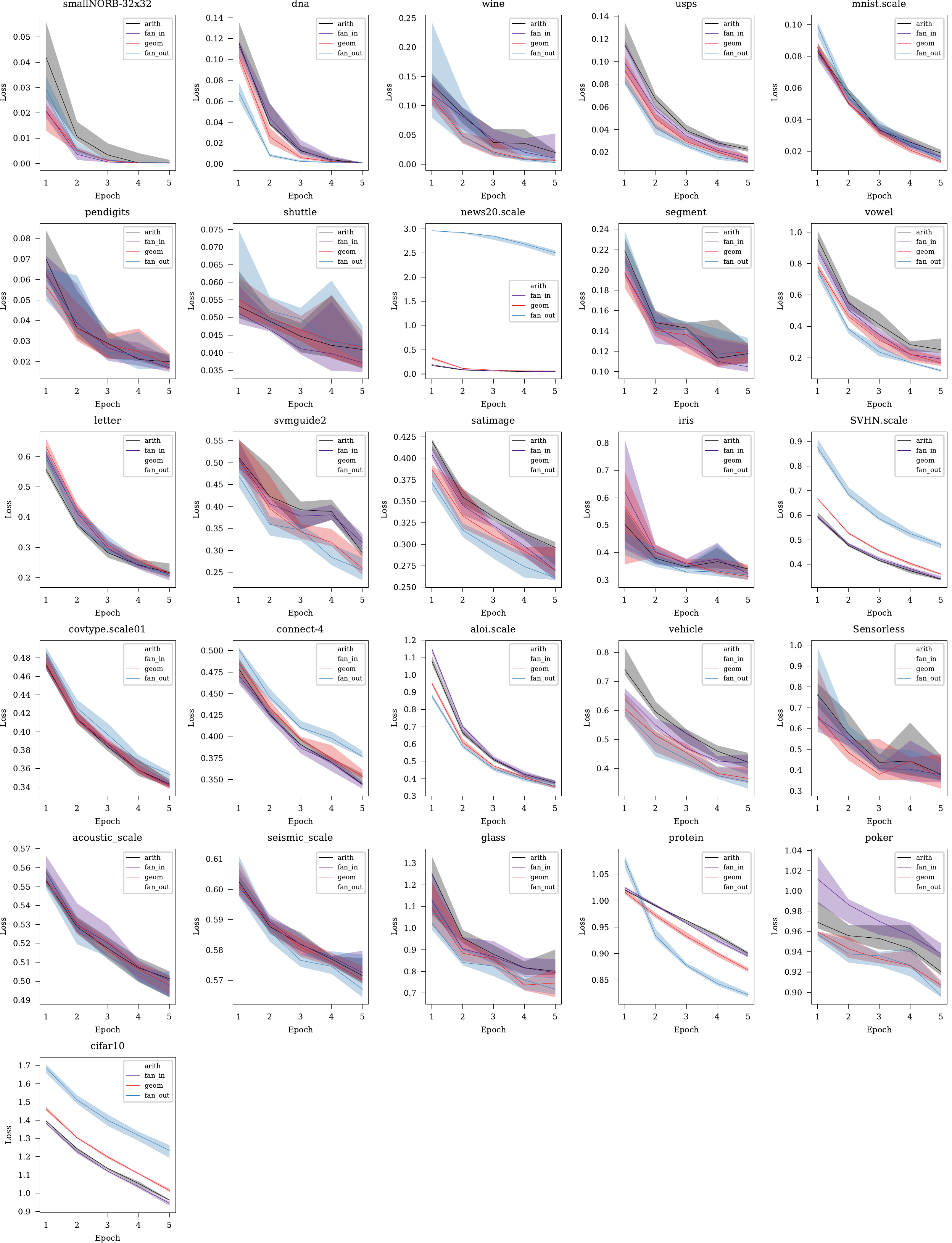}\caption{\label{fig:full-results-1}Full results for the 26 datasets. Averaging
over a large number of datasets is necessary to reduce the high variability
of the results.}
\end{figure*}

We make heavy use of the equations for forward propagation and backward
propagation of second moments, under the assumption that the weights
are uncorrelated to the activations or gradients. For a convolution
\[
y=C_{W}(x),
\]
with input channels $n_{l}$, output channels $n_{l+1},$ and square
$k\times k$ kernels, these formulas are (recall our notation for
the second moments is element-wise for vectors and matrices):

\begin{gather}
E[x_{l+1}^{2}]=n_{l}k_{l}^{2}E[W^{2}]E[x_{l}^{2}],\label{eq:conv_forward}
\end{gather}

\begin{equation}
E[\Delta x_{l+1}^{2}]=\frac{\rho_{l}^{2}E[\Delta x^{2}]}{\rho_{l+1}^{2}n_{l+1}k_{l}^{2}E[W^{2}]}.\label{eq:conv_backward}
\end{equation}

\section{Scaling properties of common operations}

Recall the scaling factor $\varsigma$:
\[
\varsigma_{l}=n_{l}\rho_{l}^{2}E[\Delta x_{l}^{2}]E[x_{l}^{2}];
\]
we show in this section how common neural network building blocks
effect this factor.

\subsection{Convolutions}

Recall the rules for forward and back-propagated second moments for
randomly initialized convolutional layers:

\begin{gather*}
E[x_{l+1}^{2}]=n_{l}k_{l}^{2}E[W^{2}]E[x_{l}^{2}],
\end{gather*}

\[
E[\Delta x_{l+1}^{2}]=\frac{\rho_{l}^{2}E[\Delta x^{2}]}{\rho_{l+1}^{2}n_{l+1}k_{l}^{2}E[W^{2}]}.
\]
These relations require that $W$ be initialized with a symmetric
mean zero distribution. When applied to the scaling factor we see
that:

\begin{align*}
\varsigma_{l+1} & =n_{l+1}\rho_{l+1}^{2}E[\Delta x_{l+1}^{2}]E[x_{l+1}^{2}]\\
 & =n_{l+1}\rho_{l+1}^{2}\frac{\rho_{l}^{2}E[\Delta x^{2}]}{\rho_{l+1}^{2}n_{l+1}k^{2}E[W^{2}]}n_{l}k^{2}E[W^{2}]E[x_{l}^{2}]\\
 & =n_{l}\rho_{l}^{2}E[\Delta x^{2}]E[x_{l}^{2}]\\
 & =\varsigma_{l}.
\end{align*}

\subsection{Linear layers}

These are a special case of convolutions with $k=1$ and $\rho=1$.

\subsection{Averaging pooling}

If the kernel size is equal to the stride, then element-wise we have:
\[
\Delta x_{l}=\frac{1}{k^{2}}\Delta x_{l+1},
\]

So:
\[
E\left[\Delta x_{l}^{2}\right]=\frac{1}{k^{4}}E\left[\Delta x_{l+1}^{2}\right].
\]

\subsection{ReLU}

If the input to a ReLU is centered and symmetrically distributed,
then during the forward pass, half of the inputs are zeroed out in
expectation, meaning that $E[x_{l+1}^{2}]=\frac{1}{2}E[x_{l}^{2}]$.
The backward operation just multiplies $\Delta x_{l+1}$ by the zero
pattern used during the forward pass, so it also zeros half of the
entries, giving $E[\Delta x_{l}^{2}]=\frac{1}{2}E[\Delta x_{l+1}^{2}]$
So:
\begin{align*}
\varsigma_{l+1} & =n_{l+1}\rho_{l+1}^{2}E[\Delta x_{l+1}^{2}]E[x_{l+1}^{2}]\\
 & =n_{l}\rho_{l}^{2}E[\Delta x_{l+1}^{2}]E[x_{l+1}^{2}]\\
 & =n_{l}\rho_{l}^{2}2E[\Delta x_{l}^{2}]\frac{1}{2}E[x_{l}^{2}]\\
 & =\varsigma_{l}.
\end{align*}

\subsection{Dropout}

The reasoning for dropout is essentially the same as for the ReLU.
If nodes are dropped out with probability $p$, then $E[x_{l+1}^{2}]=(1-p)E[x_{l}^{2}]$
and $E[\Delta x_{l}^{2}]=\left(1-p\right)E[\Delta x_{l+1}^{2}]$.
So scaling is maintained. Note that in the PyTorch implementation,
during training the outputs are further multiplied by $1/(1-p)$.

\subsection{Scalar multipliers}

Consider a layer:
\[
x_{l+1}=u_{l}x_{l}.
\]
Then
\[
E[x_{l+1}^{2}]=u_{l}^{2}E[x_{l}^{2}],
\]
and
\[
E[\Delta x_{l}^{2}]=u_{l}^{2}E[\Delta x_{l}^{2}],
\]
so the forward and backward signals are multiplied by $u_{l}^{2}$,
which maintains scaling.

\subsection{Residual blocks}

Consider a residual block of the form:
\[
x_{l+1}=x_{l}+R(x_{l}),
\]
for some operation $R$. Suppose that $E[R(x)^{2}]=s^{2}E[x^{2}]$
for some constant $s$. The forward signal second moment gets multiplied
by $\left(1+s^{2}\right)$ after the residual block:
\[
E[x_{l+1}^{2}]=E[x_{l}^{2}]\left(1+s^{2}\right).
\]
The backwards signal second moment is also multiplied by $s$:
\[
E[\Delta x_{l}^{2}]=\left(1+s^{2}\right)E[\Delta x_{l+1}^{2}].
\]
So:
\begin{align*}
\varsigma_{l+1} & =E[\Delta x_{l+1}^{2}]E[x_{l+1}^{2}]\\
 & =\frac{1}{\left(1+s^{2}\right)}E[\Delta x_{l}^{2}]E[x_{l}^{2}]\left(1+s^{2}\right)\\
 & =\varsigma_{l}.
\end{align*}

\section{Extrinsic and intrinsic form equivalence}

Recall the extrinsic definition of $\gamma_{l}$:
\[
\varsigma_{l}=n_{l}\rho_{l}^{2}E[\Delta x_{l}^{2}]E[x_{l}^{2}].
\]
We rewrite this as:
\[
n_{l}\rho_{l}^{2}E[x_{l}^{2}]=\frac{\varsigma_{l}}{E[\Delta x_{l}^{2}]}.
\]
Then by using the forward and backward relations Equation \ref{eq:conv_forward}
and Equation \ref{eq:conv_backward}:
\begin{align*}
\gamma_{l} & =n_{l}^{\text{in}}k_{l}^{2}\rho_{l+1}^{2}E\left[x_{l}^{2}\right]^{2}\frac{E[\Delta x_{l+1}^{2}]}{E[x_{l+1}^{2}]}\\
 & =k_{l}^{2}E\left[x_{l}^{2}\right]\frac{\rho_{l+1}^{2}\varsigma_{l}E[\Delta x_{l+1}^{2}]}{\rho_{l}^{2}E[\Delta x_{l}^{2}]E[x_{l+1}^{2}]}\quad\text{(\ensuremath{\varsigma} substitution)}\\
 & =k_{l}^{2}E\left[x_{l}^{2}\right]\frac{\rho_{l+1}^{2}\varsigma_{l}E[\Delta x_{l+1}^{2}]}{\rho_{l}^{2}E[\Delta x_{l}^{2}]n_{l}k_{l}^{2}E[W_{l}^{2}]E[x_{l}^{2}]}\quad\text{(forward)}\\
 & =E[\Delta x_{l+1}^{2}]\cdot\frac{\rho_{l+1}^{2}\varsigma_{l}}{\rho_{l}^{2}E[\Delta x_{l}^{2}]n_{l}E[W_{l}^{2}]}\\
 & =\frac{\rho_{l}^{2}E[\Delta x_{l}^{2}]}{\rho_{l+1}^{2}n_{l+1}k^{2}E[W_{l}^{2}]}\cdot\frac{\rho_{l+1}^{2}\varsigma_{l}}{\rho_{l}^{2}E[\Delta x_{l}^{2}]n_{l}E[W_{l}^{2}]}\quad\text{(back)}\\
 & =\frac{\varsigma_{l}}{n_{l+1}n_{l}k^{2}E[W_{l}^{2}]^{2}}.
\end{align*}

\section{The Weight gradient ratio is equal to GR scaling for MLP models}
\begin{prop}
\label{sec:variance-ratio} The weight-gradient ratio $\nu_{l}$ is
equal to the scaling $\gamma_{l}$ factor under the assumptions of
Section \ref{sec:assumptions} \@.
\end{prop}
\begin{proof}
First, we rewrite Equation \ref{eq:conv_forward} to express $E[W^{2}]$
in terms of forward moments: 
\[
E[x_{l+1}^{2}]=n_{l}k_{l}^{2}E[W^{2}]E[x_{l}^{2}],
\]
\[
\therefore E[W^{2}]=\frac{E[x_{l+1}^{2}]}{n_{l}k_{l}^{2}E[x_{l}^{2}]}.
\]
For the gradient w.r.t to weights of a convolutional layer, we have:
\[
E[\Delta W_{l}^{2}]=\rho_{l}^{2}E[x_{l}^{2}]E[\Delta x_{l+1}^{2}].
\]
Therefore:
\begin{align*}
\nu_{l} & =\frac{E[\Delta W_{l}^{2}]}{E[W_{l}^{2}]}\\
 & =n_{l}^{\text{in}}k_{l}^{2}E\left[x_{l}^{2}\right]^{2}\frac{E[\Delta x_{l+1}^{2}]}{E[x_{l+1}^{2}]}\\
 & =\gamma_{l}.
\end{align*}
\end{proof}

\section{Scaling of scalar multipliers}

Consider the layer:
\[
x_{l+1}=u_{l}x_{l}
\]
with a single learnable scalar $u_{l}$. Using :
\[
E[x_{l+1}^{2}]=E\left[u_{l}^{2}\right]E[x_{l}^{2}],
\]
we have:
\[
E\left[u_{l}^{2}\right]=\frac{E[x_{l+1}^{2}]}{E[x_{l}^{2}]}.
\]
Likewise from the equations of back-prop we have:
\[
\Delta u=\sum_{c}^{n_{l}}\sum_{i}^{\rho_{l}}\sum_{j}^{\rho_{l}}\Delta x_{l+1,c,i,j}x_{l,c,i,j},
\]
so
\[
E[\Delta u^{2}]=n_{l}\rho_{l}^{2}E[\Delta x_{l+1}^{2}]E[x_{l}^{2}].
\]
Therefore:
\begin{equation}
\nu_{l}=\frac{E[\Delta u_{l}^{2}]}{E[u_{l}^{2}]}=n_{l}\rho_{l}^{2}E[x_{l}^{2}]^{2}\frac{E[\Delta x_{l+1}^{2}]}{E[x_{l+1}^{2}]}.\label{eq:u-scaling}
\end{equation}
This equation is the same for a 1x1 convolutional layer. We can write
the scaling factor in terms of the weight $u$ more directly, by rearranging
the scaling rule as:
\[
\frac{\varsigma_{l}}{E[\Delta x_{l}^{2}]E[x_{l}^{2}]}=n_{l}\rho_{l}^{2},
\]
and substituting it into Equation \ref{eq:u-scaling}: 
\begin{align*}
n_{l}\rho_{l}^{2}E[x_{l}^{2}]^{2}\frac{E[\Delta x_{l+1}^{2}]}{E[x_{l+1}^{2}]} & =\frac{\varsigma_{l}}{E[\Delta x_{l}^{2}]E[x_{l}^{2}]}E[x_{l}^{2}]^{2}\frac{E[\Delta x_{l+1}^{2}]}{E[x_{l+1}^{2}]}\\
 & =\frac{\varsigma_{l}}{E[\Delta x_{l}^{2}]}E[x_{l}^{2}]\frac{E[\Delta x_{l+1}^{2}]}{E[x_{l+1}^{2}]}\\
 & =\frac{\varsigma_{l}}{E[\Delta x_{l}^{2}]}E[x_{l}^{2}]\frac{E[\Delta x_{l}^{2}]}{E[u_{l}^{2}]^{2}E[x_{l}^{2}]}\\
 & =\frac{\varsigma_{l}}{E[u_{l}^{2}]^{2}}.
\end{align*}

To ensure that the scaling factor matches that of convolutions used
in the network, suppose that geometric initialization is used with
global constant $c$, then each convolution has $\gamma_{l}=\varsigma_{l}/c^{2}$,
so we need:

\[
E[u_{l}^{2}]=c,
\]
or just $u_{l}=\sqrt{c}$. Notice that if we had used a per-channel
scalar instead, then the layer scaling would not match the scaling
of convolution weights, which would result in uneven layer scaling
in the network. This motivates using scalar rather than channel-wise
scaling factors. 

\section{The Gauss-Newton matrix}

\label{subsec:gauss-newton}Standard ReLU classification and regression
networks have a particularly simple structure for the Hessian with
respect to the input, as the network's output is a piecewise-linear
function $g$ feed into a final layer consisting of a convex log-softmax
operation, or a least-squares loss. This structure results in the
Hessian with respect to the input being equivalent to its \emph{Gauss-Newton}
approximation. The Gauss-Newton matrix can be written in a factored
form, which is used in the analysis we perform in this work. We emphasize
that this is just used as a convenience when working with diagonal
blocks, the GN representation is not an approximation in this case.

The (Generalized) Gauss-Newton matrix $G$ is a positive semi-definite
approximation of the Hessian of a non-convex function $f$, given
by factoring $f$ into the composition of two functions $f(x)=h(g(x))$
where $h$ is convex, and $g$ is approximated by its Jacobian matrix
$J$ at $x$, for the purpose of computing $G$:
\[
G=J^{T}\left(\nabla^{2}h(g(x))\right)J.
\]
The GN matrix also has close ties to the Fisher information matrix
\citep{martens-insights}, providing another justification for its
use.

Surprisingly, the Gauss-Newton decomposition can be used to compute
diagonal blocks of the Hessian with respect to the weights $W_{l}$
as well as the inputs \citep{martens-insights}. To see this, note
that for any activation $y_{l}$, the layers above may be treated
in a combined fashion as the $h$ in a $f(W_{l})=h(g(W_{l}))$ decomposition
of the network structure, as they are the composition of a (locally)
linear function and a convex function and thus convex. In this decomposition
$g(W_{l})=W_{l}x_{l}+b_{l}$ is a function of $W_{l}$ with $x_{l}$
fixed, and as this is linear in $W_{l}$, the Gauss-Newton approximation
to the block is thus not an approximation.

\section{GR scaling derivation\label{sec:GR-scaling-derivation}}

Our quantity of interest is the average squared singular value of
$G_{l}$, which is simply equal to the (element-wise) non-central
second moment of the product of $G$ with a i.i.d normal random vector
$r$:
\[
E[\left(G_{l}r\right)^{2}]=E[\left(J_{l}^{T}R_{l}J_{l}r\right)^{2}].
\]

Recall that our notation $E[X^{2}]$ refers to the element-wise non-central
second moment of the vector. To compute the second moment of the elements
of $G_{l}r$, we can calculate the second moment of matrix-random-vector
products against $J_{l}$, $R_{l}$ and $J_{l}^{T}$ separately since
$R$ is uncorrelated with $J_{l}$, and the back-propagated gradient
$\Delta y_{l}$ is uncorrelated with $y_{l}$ (Assumption A3). 

\subsection*{Jacobian products $J_{l}$ and $J_{l}^{T}$}

Note that each row of $J_{l}$ has $n_{l}^{\text{in}}$ non-zero elements,
each containing a value from $x_{l}$. This structure can be written
as a block matrix,
\begin{equation}
J_{l}=\left[\begin{array}{ccc}
x_{l} & 0 & 0\\
0 & x_{l} & 0\\
0 & 0 & \ddots
\end{array}\right],\label{eq:j-matrix}
\end{equation}
Where each $x_{l}$ is a $1\times n_{l}^{\text{in}}$ row vector.
This can also be written as a Kronecker product with an identity matrix
as $I_{n_{l}^{\text{out}}}\otimes x_{l}$. The value $x_{l}$ is i.i.d
random at the bottom layer of the network. For layers further up,
the multiplication by a random weight matrix from the previous layer
ensures that the entries of $x_{l}$ are identically distributed.
So we have:
\begin{equation}
E\left[\left(J_{l}r\right)^{2}\right]=n_{l}^{\text{in}}E[r^{2}]E[x_{l}^{2}]=n_{l}^{\text{in}}E[x_{l}^{2}].\label{eq:j-eq}
\end{equation}
Note that we didn't assume that the input $x_{l}$ is mean zero, so
$Var[x_{l}]\neq E[x_{l}^{2}].$ This is needed as often the input
to a layer is the output from a ReLU operation, which will not be
mean zero.

For the transposed case, we have a single entry per column, so when
multiplying by an i.i.d random vector $u$ we have: 
\begin{equation}
E\left[\left(J_{l}^{T}u\right)^{2}\right]=E[u^{2}]E[x_{l}^{2}].\label{eq:jt-eq}
\end{equation}

\subsection*{Upper Hessian $R_{l}$ product}

Instead of using $R_{l}u$, for any arbitrary random $u$, we will
instead compute it for $u=y_{l}/E[y_{l}^{2}]$, it will have the same
expectation since both $J_{l}r$ and $y_{l}$ are uncorrelated with
$R_{l}$. The piecewise linear structure of the network above $y_{l}$
with respect to the $y_{l}$ makes the structure of $R_{l}$ particularly
simple. It is a least-squares problem $g(y_{l})=\frac{1}{2}\left\Vert \Phi y_{l}-t\right\Vert ^{2}$
for some $\Phi$ that is the linearization of the remainder of the
network. The gradient is $\Delta y=\Phi^{T}\left(\Phi y-t\right)$
and the Hessian is simply $R=\Phi^{T}\Phi$. So we have that
\begin{align*}
E\left[\Delta y_{l}^{2}\right] & =E\left[\frac{1}{n_{l}^{\text{out}}}\left\Vert \Phi^{T}\left(\Phi y-t\right)\right\Vert ^{2}\right]\\
 & =E\left[\frac{1}{n_{l}^{\text{out}}}\left\Vert \Phi^{T}\Phi y\right\Vert ^{2}\right]+E\left[\frac{1}{n_{l}^{\text{out}}}\left\Vert \Phi^{T}t\right\Vert ^{2}\right]\\
 & =E\left[\frac{1}{n_{l}^{\text{out}}}\left\Vert \Phi^{T}\Phi y\right\Vert ^{2}\right]+O\left(\frac{1}{n_{l}^{\text{out}}}\right).\\
 & =E\left[\left(R_{l}y_{l}\right)^{2}\right]+O\left(\frac{1}{n_{l}^{\text{out}}}\right).
\end{align*}
Applying this gives:
\begin{align}
E\left(R_{l}u\right)^{2} & =E[u^{2}]E[\left(R_{l}y_{l}\right)^{2}]/E[y_{l}^{2}]\label{eq:r-eq}\\
 & =E[u^{2}]E[\Delta y_{l}^{2}]/E[y_{l}^{2}]+O\left(\frac{E[u^{2}]}{n_{l}^{\text{out}}E[y_{l}^{2}]}\right)
\end{align}

\subsection*{Combining}

To compute $E[\left(G_{l}r\right)^{2}]=E[\left(J_{l}^{T}R_{l}J_{l}r\right)^{2}]$
we then combine the simplifications from Equations \ref{eq:j-eq},
\ref{eq:jt-eq}, and \ref{eq:r-eq} to give:
\[
E[\left(G_{l}r\right)^{2}]=n_{l}^{\text{in}}E\left[x_{l}^{2}\right]^{2}\frac{E[\Delta y_{l}^{2}]}{E[y_{l}^{2}]}+O\left(\frac{n_{l}^{\text{in}}E\left[x_{l}^{2}\right]^{2}}{n_{l}^{\text{out}}E[y_{l}^{2}]}\right).
\]

\section{Details of LibSVM dataset input/output scaling }

To prevent the results from being skewed by the number of classes
and the number of inputs affecting the output variance, the logit
output of the network was scaled to have standard deviation 0.05 after
the first minibatch evaluation for every method, with the scaling
constant fixed thereafter. LayerNorm was used on the input to whiten
the data. Weight decay of 0.00001 was used for every dataset. To aggregate
the losses across datasets we divided by the worst loss across the
initializations before averaging.
\end{document}